\documentclass[11pt]{article}

\usepackage[utf8]{inputenc}
\usepackage[T1]{fontenc}
\usepackage{lmodern}
\usepackage{microtype}
\usepackage{geometry}
\usepackage{amsmath,amssymb,amsthm,mathtools}
\usepackage{booktabs,tabularx,array}
\usepackage[round,authoryear]{natbib}
\usepackage[hidelinks]{hyperref}
\usepackage{xcolor}
\usepackage{listings}
\usepackage{tikz}
\usetikzlibrary{arrows.meta,positioning,shapes.geometric}

\lstdefinestyle{sftpython}{
  language=Python,
  basicstyle=\ttfamily\small,
  keywordstyle=\bfseries,
  commentstyle=\itshape\color{gray!70!black},
  stringstyle=\color{gray!60!black},
  showstringspaces=false,
  columns=fullflexible,
  breaklines=true,
  frame=single,
  rulecolor=\color{gray!40},
  numbers=left,
  numberstyle=\tiny\color{gray!70!black},
  xleftmargin=2em,
  framexleftmargin=1.5em
}

\newtheorem{definition}{Definition}
\newtheorem{proposition}{Proposition}

\newtheorem{remark}{Remark}

\DeclareMathOperator*{\argmin}{arg\,min}
\DeclareMathOperator{\softplus}{softplus}
\DeclareMathOperator{\diag}{diag}

\DeclareMathOperator{\pool}{pool}

\title{Semantic Field Theory:\\
Historical Origin, Higher-Order Interaction, and Stabilized Semantic Inference}

\author{%
Dimitris Vartziotis$^{1,2}$\thanks{Corresponding author: \texttt{dimitris.vartziotis@nikitec.gr}}\\
\small $^1$NIKI -- Digital Engineering, Ioannina, Greece\\
\small $^2$TWT Science \& Innovation, Stuttgart, Germany
}

\date{May 2026}

\begin{document}

\maketitle

\begin{abstract}
Semantic Field Theory (SFT) has developed from a philosophical critique of strong anti-formalist readings of language games into a proposed computational model class for lexical semantics, higher order composition, and stabilized interpretation. This paper reconstructs that evolution and gives SFT a sharper mathematical core suitable for independent evaluation in computational linguistics and representation learning. The central proposal is that a tractable level of linguistic organization can be modeled through lexical representations expressed as semantic fields, through contextual deformation of those fields, through interaction terms defined over subsets of tokens, and through stabilization governed by semantic energy dynamics. The paper contributes five formal elements. First, it defines a semantic field model as a tuple consisting of a semantic space, a lexical field lifting, a contextual deformation map, an interaction complex, and an interpretation functional. Second, it proves a Gaussian product closure result showing that multiplicative field interactions have explicit centers, precisions, and compatibility factors. Third, it generalizes the three-word problem by using Mobius inversion on the subset lattice to isolate irreducible semantic interactions of arbitrary order. Fourth, it introduces an order spectrum that measures how much field mass is explained at each interaction order. Fifth, it formulates stabilized interpretation as minimization of an energy functional associated with the sentence and gives existence, descent, and stability conditions. A small worked example shows how a three-word summer day triple can be represented by Gaussian semantic fields, implemented in Python, and summarized by a flow diagram. The result is not a completed theory of natural language meaning and does not replace social, pragmatic, or normative accounts of language. It is a mathematically explicit hypothesis about one representational level: how public language use may give rise to stable geometric, interactional, and dynamical regularities that can be estimated, ablated, and compared with transformer models.
\end{abstract}

\section{Introduction}

The mathematical study of linguistic meaning has passed through several regimes. The distributional hypothesis treated lexical meaning as a pattern of contextual occurrence \citep{harris1954,firth1957}. Vector-space semantics made that hypothesis geometrical and computational \citep{turney2010}. Compositional distributional semantics then asked how phrase and sentence meanings could be built from lexical representations \citep{mitchell2008,mitchell2010,coecke2010,grefenstette2011}. Neural embeddings made dense semantic geometry empirically powerful \citep{mikolov2013,pennington2014}. Contextual encoders and transformers shifted the unit of representation from a word type to a token embedded in an utterance \citep{peters2018,devlin2019,vaswani2017}. Large language models subsequently showed that broad linguistic regularities can be optimized, transferred, and used for general-purpose language behavior at scale \citep{kaplan2020,brown2020,bommasani2021}.

This development motivates a question that is both scientific and philosophical. Are mathematical structures merely convenient external representations of language, or do they reveal a stable level of organization generated by language use itself? A strong answer would be metaphysical and is unnecessary for the present paper. A weaker and scientifically useful answer is enough: public language use appears to induce regularities that can be represented as geometry, interaction, and dynamics. These regularities do not settle questions of reference, grounding, pragmatics, or normativity \citep{bender2020,bender2021}. They nevertheless form a legitimate object of computational theory.

Semantic Field Theory (SFT) is proposed as one such theory. From its earliest philosophical stages, SFT was guided by the intuition that lexical meaning behaves more like a distributed semantic field than a discrete symbolic atom or isolated vector representation. A lexical item should therefore not be represented only as a point, a symbol, or a single vector, but as a distributed field over a semantic space; an utterance should not be represented only as a sum of lexical items, but as an interaction field over subsets of tokens; and interpretation should not be treated only as retrieval of a static object, but as stabilization in a sentence-conditioned energy landscape. Earlier forms of this intuition were developed in philosophical work on Wittgenstein and language games \citep{vartziotis2012,vartziotis2017}. A later preprint formulated the explicit contrast between language-game anti-formalism and language as mathematical structure \citep{vartziotis2026a}. The present paper takes the next step: it reconstructs SFT as an evolving formal program and supplies a mathematically consolidated model class.

The contribution is not only historical. The paper introduces several technical elements that make SFT more mature as a computational hypothesis. First, it defines a \emph{semantic field model} as a tuple of maps and function spaces. Second, it replaces the merely metaphorical notion of field interaction by a Gaussian field algebra, including a closed-form product rule for interaction centers and compatibility factors. Third, it generalizes the three-word problem through Mobius inversion on the subset lattice. Fourth, it introduces an \emph{order spectrum}, a measurable profile of first-order, pairwise, third-order, and higher semantic residuals. Fifth, it defines a stabilized interpretation map by energy minimization and gives basic existence, descent, and perturbation stability results. A short worked example is included to make the formal machinery executable and inspectable in the simple case of a three-token utterance.

The intended status of the paper is therefore precise. SFT is not claimed to be a full semantics of natural language. It is a formal model class for a particular level of semantic organization: distributed lexical influence, context-dependent deformation, higher-order interaction, and stabilized inference. This level is compatible with use-theoretic and social accounts of meaning, but it is not exhausted by them. Its value should be judged by whether it yields estimable parameters, interpretable diagnostics, and discriminable empirical consequences.

\paragraph{Organization.}
Section~\ref{sec:history} reconstructs the historical development of SFT. Section~\ref{sec:positioning} clarifies the relation between language games, distributional structure, and formal semantics. Section~\ref{sec:model} defines the semantic field model. Section~\ref{sec:algebra} develops the Gaussian interaction algebra. Section~\ref{sec:example} gives a minimal computational example using a three-word summer-day triple. Section~\ref{sec:mobius} gives the residual theory of higher-order composition. Section~\ref{sec:stabilization} treats interpretation as energy minimization. Section~\ref{sec:transformers} relates SFT to transformer representations. Section~\ref{sec:evaluation} proposes empirical tests. Section~\ref{sec:limits} states limitations.

\section{Historical evolution of SFT}
\label{sec:history}

SFT did not originate as a response to current large language models. Its earliest motivation was philosophical: a resistance to the inference that because meaning is public, social, and use-governed, the search for inner semantic structure must be illegitimate. The development can be summarized as a sequence of increasingly explicit formulations.

\begin{table}[t]
\centering
\small
\begin{tabularx}{\linewidth}{@{}p{0.17\linewidth}p{0.29\linewidth}X@{}}
\toprule
Stage & Main source & Role in the evolution of SFT \\
\midrule
Philosophical origin & \citet{vartziotis2012} & Critique of restrictive readings of Wittgenstein and initial reflections on semantic organization beyond local use descriptions. \\
Intermediate formulation & \citet{vartziotis2017} & German commentary stage in which the opposition between language games and structural semantic regularity becomes more explicit. \\
Language-game confrontation & \citet{vartziotis2026a} & Explicit framing of SFT against strong anti-formalism, with LLMs used as evidence that language use leaves tractable mathematical structure. \\
Philosophical-semantic synthesis & In press: \emph{Wittgenstein and the End of Language Games} & Development of SFT as a mathematical semantics framework connected to predictive language models and post-language-game philosophy. \\
LLM-era philosophical expansion & In press: \emph{The Revolution of LLMs} & Exploration of the philosophical and mathematical implications of large language models and their relation to semantic-field structures. \\
Current formal program & Present paper & Mathematical consolidation: field lifting, contextual deformation, interaction complexes, Mobius residuals, order spectra, and energy-based stabilization. \\
\bottomrule
\end{tabularx}
\caption{A compact reconstruction of the historical development of SFT. The earlier works are treated here as conceptual antecedents, not as complete technical implementations of the present model.}
\label{tab:evolution}
\end{table}

The first stage is a philosophical orientation. Later Wittgenstein emphasizes language as public practice, rule-following, and participation in language games \citep{wittgenstein1953}. This remains a necessary constraint on any account of meaning. However, the stronger claim that this public character prohibits formal internal structure does not follow. The early SFT intuition was that use and structure should be separated analytically: use supplies the public ground of meaning, while repeated use can also produce stable patterns that are mathematically articulable.

Already in the earlier philosophical stages of SFT, a distinction emerged between lexical fields and linguistic fields, although not yet in formal mathematical language. Lexical fields correspond to distributed semantic potential associated with word types. Linguistic fields correspond to utterance-level interaction systems produced when lexical fields are activated together. This distinction is central because it prevents SFT from collapsing into one-vector-per-word semantics. It also prevents the theory from identifying sentence meaning with simple aggregation.

The third stage is the LLM-era reinterpretation. Transformer-based systems do not prove that language is reducible to vectors, nor do they solve the philosophical problem of meaning. They do, however, demonstrate that large-scale language use contains stable, learnable, high-dimensional regularities. This observation weakens strong anti-structural prohibitions and motivates a more detailed formal account. This broader reinterpretation is further developed in two books currently in press at Literareon -- Utz Verlag by the author: \emph{Wittgenstein and the End of Language Games: Mathematical Semantics, Field Theory, and the Age of Predictive Language Models}, and \emph{The Revolution of LLMs: Philosophical and Mathematical Conjugations}. These works extend the philosophical implications of SFT beyond the present computational formulation and examine how predictive language models reshape traditional debates concerning meaning, structure, and linguistic representation.

The present paper is the fourth stage. It treats SFT not as a slogan but as a family of estimable mathematical models. The historical claim is modest: SFT began as a philosophical critique and is now reformulated as a computational hypothesis. The scientific claim is also modest: if the hypothesis is meaningful, it should generate quantities that can be estimated and ablated, such as field overlap, residual interaction mass, energy gaps, and stabilization trajectories.

\section{Use, structure, and levels of explanation}
\label{sec:positioning}

A recurring source of confusion is the phrase ``inner structure''. If it means private mental objects that determine meaning independently of public criteria, then SFT does not require it. If it means mathematically describable regularities induced by public language use, then rejecting it is unnecessary and scientifically costly.

We therefore distinguish four levels.

\begin{enumerate}
\item \textbf{Normative-pragmatic level.} Meaning is learned, corrected, and stabilized in public linguistic practice. This is the level emphasized by language-game approaches \citep{wittgenstein1953,kripke1982}.
\item \textbf{Distributional level.} Repeated use leaves statistical regularities in contexts, co-occurrences, substitutions, and entailment patterns \citep{harris1954,firth1957,turney2010}.
\item \textbf{Geometric-field level.} Distributional regularities can be represented as fields, distances, overlaps, displacements, and interaction structures in a semantic space.
\item \textbf{Dynamical-computational level.} Interpretation can be modeled as a process that evolves toward stable states under a sentence-conditioned field.
\end{enumerate}

SFT is a theory of the third and fourth levels. It is constrained by the first and informed by the second. This layered view is important for scientific maturity because it avoids two reductions. It avoids reducing meaning to social practice alone, and it avoids reducing meaning to hidden vectors alone. The object of SFT is narrower: the structured representational regularities through which lexical and utterance-level meanings can be approximated, composed, and dynamically stabilized.

\section{Semantic Field Theory as a model class}
\label{sec:model}

Let $\mathbf{w}=(w_1,\ldots,w_m)$ be a token sequence and let $[m]=\{1,\ldots,m\}$. Let $S\subseteq \mathbb{R}^n$ be a semantic space with Euclidean norm $\|\cdot\|$. The space is not an ontological container of meanings; it is a formal domain for semantic proximity, deformation, interaction, and dynamics.

\begin{definition}[Semantic field model]
A semantic field model of order $p$ is a tuple
\[
\mathfrak{M}_p=(S,\mathcal{V},\ell_\Theta,D_\Theta,\mathcal{K}_p,\Lambda_\Theta,U_\Theta),
\]
where $S\subseteq\mathbb{R}^n$ is a semantic space, $\mathcal{V}$ is a vocabulary, $\ell_\Theta$ maps lexical types to base fields, $D_\Theta$ maps base fields to contextually deformed token fields, $\mathcal{K}_p(\mathbf{w})$ is an interaction complex containing subsets of $[m]$ of size at most $p$, $\Lambda_\Theta$ assigns sentence-conditioned interaction coefficients, and $U_\Theta(\cdot\mid\mathbf{w})$ is an interpretation energy.
\end{definition}

This definition separates five roles that were often implicit in earlier discussions: lexical structure, contextual deformation, interaction selection, interaction weighting, and interpretation. Each role may be instantiated by symbolic, neural, probabilistic, or hybrid mechanisms.

\subsection{Lexical field lifting}

Let each vocabulary type $t\in\mathcal{V}$ have a base embedding $z_t\in\mathbb{R}^{d_z}$. A lexical lifting map sends $z_t$ to field parameters
\begin{align}
 c_t &= W_c z_t, \\
 a_t &= \softplus(w_a^\top z_t+b_a), \\
 \Sigma_t &= \diag\big(\softplus(W_\Sigma z_t)+\delta\mathbf{1}\big), \qquad \delta>0.
\end{align}
For a token occurrence $w_i=t_i$, the base lexical field is
\begin{equation}
L_i(x)=a_{t_i}\exp\!\left[-\frac12(x-c_{t_i})^\top\Sigma_{t_i}^{-1}(x-c_{t_i})\right],\qquad x\in S.
\label{eq:base-lexical-field}
\end{equation}
A Gaussian field is not required by the theory, but it gives a useful base case: centers encode dominant regions, covariances encode semantic spread and anisotropy, and amplitudes encode field strength.

\subsection{Contextual deformation}

Type-level fields must be distinguished from token-level fields. Let $q_\phi(\mathbf{w})$ be a contextual sentence representation, possibly produced by a transformer encoder, recurrent model, syntactic encoder, or task-specific context map. A contextually deformed token field is
\begin{equation}
\widetilde{L}_i(x\mid\mathbf{w})=\tilde a_i(\mathbf{w})\exp\!\left[-\frac12(x-\tilde c_i(\mathbf{w}))^\top\widetilde\Sigma_i(\mathbf{w})^{-1}(x-\tilde c_i(\mathbf{w}))\right],
\label{eq:deformed-field}
\end{equation}
with
\begin{align}
\tilde a_i(\mathbf{w}) &= a_{t_i}\,\gamma_i(\mathbf{w}),\qquad \gamma_i(\mathbf{w})>0, \\
\tilde c_i(\mathbf{w}) &= c_{t_i}+d_i(\mathbf{w}), \\
\widetilde\Sigma_i(\mathbf{w}) &= \Sigma_{t_i}+R_i(\mathbf{w}), \qquad \widetilde\Sigma_i(\mathbf{w})\succ0.
\end{align}
The simplest version sets $d_i=0$ and $R_i=0$ and uses only a non-negative gate $\gamma_i$. A richer version permits contextual displacement and covariance deformation. This distinction is conceptually important: context need not erase lexical identity; it can deform lexical influence.

A concrete neural parameterization is
\begin{align}
\gamma_i(\mathbf{w}) &= \softplus\!\left(u_g^\top\tanh(W_g[z_{t_i};q_\phi(\mathbf{w});e_i])+b_g\right), \\
d_i(\mathbf{w}) &= W_d\tanh(W_{dc}[z_{t_i};q_\phi(\mathbf{w});e_i]),
\end{align}
where $e_i$ may include positional or syntactic features. This makes the model order-sensitive even when interaction products are symmetric in their arguments.

\subsection{Interaction complexes}

The full powerset of $[m]$ is usually computationally impossible. SFT therefore separates the formal definition from the selected interaction structure.

\begin{definition}[Interaction complex]
For a sentence $\mathbf{w}$ and order $p$, an interaction complex $\mathcal{K}_p(\mathbf{w})$ is a family of non-empty subsets $A\subseteq[m]$ such that $|A|\le p$. It is downward closed if $A\in\mathcal{K}_p(\mathbf{w})$ and $\emptyset\neq B\subseteq A$ imply $B\in\mathcal{K}_p(\mathbf{w})$.
\end{definition}

A downward-closed complex is natural when residual interactions are interpreted hierarchically. In practice, $\mathcal{K}_p$ can be selected by local windows, syntactic neighborhoods, top-$k$ contextual similarity, dependency paths, or learned sparsity.

For each $A\in\mathcal{K}_p(\mathbf{w})$, define
\begin{equation}
\Phi_A(x\mid\mathbf{w})=\lambda_A(\mathbf{w})\prod_{i\in A}\widetilde L_i(x\mid\mathbf{w}),
\label{eq:interaction-field}
\end{equation}
where $\lambda_A(\mathbf{w})\in\mathbb{R}$ is an interaction coefficient. The linguistic field of order $p$ is
\begin{equation}
\mathcal{L}_p(x\mid\mathbf{w})=\sum_{A\in\mathcal{K}_p(\mathbf{w})}\Phi_A(x\mid\mathbf{w}).
\label{eq:linguistic-field}
\end{equation}
When $|A|=1$, one may set $\lambda_A=1$ or absorb singleton weights into $\widetilde L_i$. Positive coefficients reinforce joint activation; negative coefficients suppress it.

A sentence-conditioned coefficient can be parameterized by
\begin{equation}
\lambda_A(\mathbf{w})=v_{|A|}^\top \tanh\!\left(W_{|A|}[\pool(z_{t_i},e_i:i\in A);q_\phi(\mathbf{w})]\right)+b_{|A|},
\label{eq:lambda-param}
\end{equation}
where $\pool$ is symmetric if order information is carried elsewhere, or order-sensitive if the subset is represented as an ordered tuple.

\section{Gaussian field algebra}
\label{sec:algebra}

A main advantage of Gaussian lexical fields is that multiplicative interactions remain analytically tractable. This gives SFT a concrete geometry of composition rather than a purely verbal field metaphor.

For $i\in A$, write
\[
\widetilde L_i(x\mid\mathbf{w})=\tilde a_i \exp\!\left[-\frac12(x-\tilde c_i)^\top \widetilde\Sigma_i^{-1}(x-\tilde c_i)\right],
\]
where the dependence on $\mathbf{w}$ is suppressed for readability. Let $P_i=\widetilde\Sigma_i^{-1}$.

\begin{proposition}[Gaussian product closure]
For any non-empty $A\subseteq[m]$, define
\begin{align}
P_A &= \sum_{i\in A}P_i, \\
\Sigma_A &= P_A^{-1}, \\
\mu_A &= \Sigma_A\sum_{i\in A}P_i\tilde c_i, \\
\kappa_A &= \exp\!\left[-\frac12\left(\sum_{i\in A}\tilde c_i^\top P_i\tilde c_i-\mu_A^\top P_A\mu_A\right)\right].
\end{align}
Then
\begin{equation}
\prod_{i\in A}\widetilde L_i(x\mid\mathbf{w})
=\left(\prod_{i\in A}\tilde a_i\right)\kappa_A
\exp\!\left[-\frac12(x-\mu_A)^\top P_A(x-\mu_A)\right].
\label{eq:gaussian-product}
\end{equation}
\end{proposition}

\begin{proof}
Expand the exponent:
\[
\sum_{i\in A}(x-\tilde c_i)^\top P_i(x-\tilde c_i)
=x^\top P_A x-2x^\top\sum_{i\in A}P_i\tilde c_i+\sum_{i\in A}\tilde c_i^\top P_i\tilde c_i.
\]
Since $\mu_A=P_A^{-1}\sum_{i\in A}P_i\tilde c_i$, completing the square gives
\[
x^\top P_A x-2x^\top P_A\mu_A
=(x-\mu_A)^\top P_A(x-\mu_A)-\mu_A^\top P_A\mu_A.
\]
Substitution yields \eqref{eq:gaussian-product}.
\end{proof}

The term $\mu_A$ is the \emph{interaction focus}. It is a precision-weighted semantic location jointly supported by the fields in $A$. The matrix $P_A$ is the \emph{interaction precision}: adding fields sharpens the product. The scalar $\kappa_A$ is a \emph{compatibility factor}. It decreases when centers are mutually distant relative to their covariances and increases when the fields overlap. Thus, in the Gaussian case, SFT yields explicit quantities for lexical compatibility and compositional focus.

\begin{remark}[Closed-form field overlap]
For two Gaussian fields, the overlap integral $\int_{\mathbb{R}^n}L_i(x)L_j(x)\,dx$ is proportional to the compatibility factor for $A=\{i,j\}$ and includes the normalizing term $(2\pi)^{n/2}|P_i+P_j|^{-1/2}$. This quantity can be used as a field-theoretic similarity measure distinct from cosine similarity between centers.
\end{remark}

\subsection{Interaction tension}

The exponent in $\kappa_A$ defines a natural measure of semantic tension:
\begin{equation}
T_A=\sum_{i\in A}\tilde c_i^\top P_i\tilde c_i-\mu_A^\top P_A\mu_A\ge0.
\label{eq:tension}
\end{equation}
For $|A|=2$, $T_A$ reduces to a Mahalanobis-type discrepancy between the two centers under the combined covariance. For larger $A$, it measures how difficult it is for all fields in $A$ to share a common focus. A high-tension triple with a non-zero positive coefficient $\lambda_A$ is a natural formal candidate for metaphor, coercion, or non-literal construal: the fields are not simply close, yet the sentence activates a joint focus.

\section{Minimal computational example: a summer-day triple}
\label{sec:example}

This section gives a deliberately small, transparent instantiation of the preceding Gaussian algebra. The example is not intended as an empirical validation of SFT. It is a didactic bridge from the formal definitions to executable code. Consider the three-token summer-day triple
\[
\mathbf{w}=(w_1,w_2,w_3)=(\text{helios},\text{thalassa},\text{zesti}),
\]
where the transliterated Greek words correspond to \emph{sun}, \emph{sea}, and \emph{heat}. Let the semantic space be $S=\mathbb{R}^3$, with coordinates interpreted as summer salience, natural/marine salience, and thermal intensity. Assign toy centers
\begin{align}
 c_1 &= (0.95,0.70,0.90)^\top, \\
 c_2 &= (0.85,0.95,0.45)^\top, \\
 c_3 &= (0.90,0.45,0.98)^\top,
\end{align}
and use the simplest isotropic fields
\begin{equation}
P_i=I_3,\qquad a_i=1,\qquad i=1,2,3.
\end{equation}
Thus each lexical field has the form
\begin{equation}
L_i(x)=\exp\!\left[-\frac12\|x-c_i\|^2\right].
\end{equation}
For $A=\{1,2,3\}$, the SFT interaction quantities become
\begin{align}
P_A &= 3I_3, \\
\Sigma_A &= \frac13 I_3, \\
\mu_A &= \frac13(c_1+c_2+c_3)=(0.900,0.700,0.777)^\top, \\
T_A &= \sum_{i=1}^3\|c_i\|^2-3\|\mu_A\|^2 \approx 0.293, \\
\kappa_A &= \exp(-T_A/2)\approx 0.864.
\end{align}
In this toy setting, the compatibility is high because the three fields jointly support a coherent focus: a hot summer day near the sea. The second coordinate is pulled upward by \emph{thalassa}, the third coordinate is pulled upward by \emph{helios} and \emph{zesti}, and the first coordinate remains high across all three terms. This illustrates the central SFT idea that the utterance-level focus is not one lexical point, but a stabilized region produced by field interaction.

\subsection{Executable prototype}

Listing~\ref{lst:summer-triple} implements the same calculation. The centers are hand-coded only to make the algebra visible; in an empirical implementation they would be learned from data or induced from an encoder.

\begin{lstlisting}[style=sftpython,caption={Minimal SFT computation for a summer-day triple.},label={lst:summer-triple}]
import numpy as np
from itertools import combinations

# Semantic Field Theory: minimal Gaussian triple example
# Tokens: helios = sun, thalassa = sea, zesti = heat
words = ["helios", "thalassa", "zesti"]

# Toy semantic dimensions:
# [summer salience, natural/marine salience, thermal intensity]
centers = {
    "helios":   np.array([0.95, 0.70, 0.90]),
    "thalassa": np.array([0.85, 0.95, 0.45]),
    "zesti":    np.array([0.90, 0.45, 0.98]),
}

fields = []
for word in words:
    fields.append({
        "word": word,
        "center": centers[word],
        "precision": np.eye(3),
        "amplitude": 1.0,
    })


def sft_interaction(selected_fields):
    """Gaussian product interaction according to SFT."""
    P_A = sum(f["precision"] for f in selected_fields)
    Sigma_A = np.linalg.inv(P_A)

    weighted_sum = sum(
        f["precision"] @ f["center"] for f in selected_fields
    )
    mu_A = Sigma_A @ weighted_sum

    tension = sum(
        f["center"].T @ f["precision"] @ f["center"]
        for f in selected_fields
    )
    tension -= mu_A.T @ P_A @ mu_A

    compatibility = np.exp(-0.5 * tension)
    return mu_A, float(tension), float(compatibility)


print("Pairwise interactions")
for pair in combinations(fields, 2):
    mu, tension, compatibility = sft_interaction(pair)
    print([f["word"] for f in pair])
    print("  focus mu       =", np.round(mu, 3))
    print("  tension T      =", round(tension, 4))
    print("  compatibility =", round(compatibility, 4))

print("\nTriple interaction")
mu, tension, compatibility = sft_interaction(fields)
print(words)
print("  focus mu_123       =", np.round(mu, 3))
print("  tension T_123      =", round(tension, 4))
print("  compatibility k_123=", round(compatibility, 4))

if compatibility > 0.8:
    print("Interpretation: coherent summer-day semantic field.")
elif compatibility > 0.5:
    print("Interpretation: related fields with moderate tension.")
else:
    print("Interpretation: weak, unstable, or metaphorical interaction.")
\end{lstlisting}

\subsection{Flussdiagramm}

Figure~\ref{fig:summer-flow} summarizes the computational flow. The diagram is intentionally close to the mathematical pipeline: lexical items are lifted to field centers, Gaussian fields are constructed, subset interactions are evaluated, and the triple-level focus and diagnostics are interpreted.

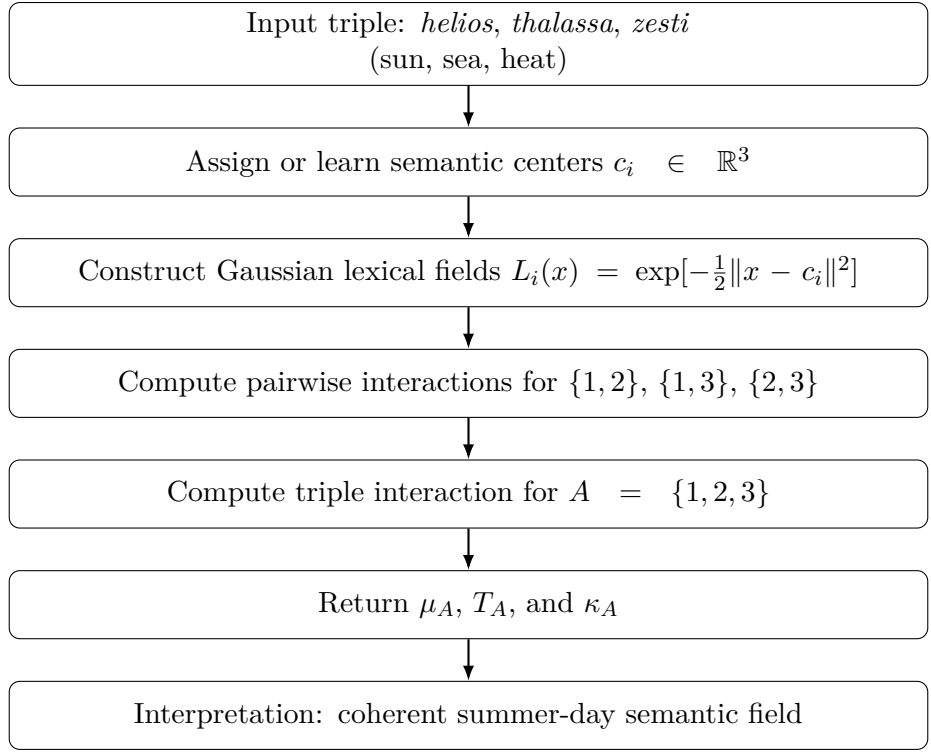
\begin{figure}[t]
\centering
\tikzset{
  sftblock/.style={draw, rounded corners, align=center, text width=0.72\linewidth, minimum height=0.9cm},
  sftarrow/.style={-{Latex[length=2mm]}, thick}
}
\begin{tikzpicture}[node distance=0.55cm]
\node[sftblock] (input) {Input triple: \emph{helios}, \emph{thalassa}, \emph{zesti}\\(sun, sea, heat)};
\node[sftblock, below=of input] (vectors) {Assign or learn semantic centers $c_i\in\mathbb{R}^3$};
\node[sftblock, below=of vectors] (fields) {Construct Gaussian lexical fields $L_i(x)=\exp[-\frac12\|x-c_i\|^2]$};
\node[sftblock, below=of fields] (pairs) {Compute pairwise interactions for $\{1,2\}$, $\{1,3\}$, $\{2,3\}$};
\node[sftblock, below=of pairs] (triple) {Compute triple interaction for $A=\{1,2,3\}$};
\node[sftblock, below=of triple] (diagnostics) {Return $\mu_A$, $T_A$, and $\kappa_A$};
\node[sftblock, below=of diagnostics] (interpretation) {Interpretation: coherent summer-day semantic field};

\draw[sftarrow] (input) -- (vectors);
\draw[sftarrow] (vectors) -- (fields);
\draw[sftarrow] (fields) -- (pairs);
\draw[sftarrow] (pairs) -- (triple);
\draw[sftarrow] (triple) -- (diagnostics);
\draw[sftarrow] (diagnostics) -- (interpretation);
\end{tikzpicture}
\caption{Flussdiagramm for the minimal SFT summer-day example.}
\label{fig:summer-flow}
\end{figure}

\section{Residual interaction and the generalized three-word problem}
\label{sec:mobius}

The three-word problem states that some semantic effects are not reconstructible from singleton and pairwise effects. The mature form of this idea is not limited to triples. It is a general residual theory over the subset lattice.

For any non-empty $A\subseteq[m]$, let $\mathcal{L}_A(x\mid\mathbf{w})$ denote the linguistic field induced by the subsequence or subconfiguration indexed by $A$. Define $\mathcal{L}_\emptyset\equiv0$. The residual interaction field is
\begin{equation}
\Delta_A(x\mid\mathbf{w})=\sum_{B\subseteq A}(-1)^{|A|-|B|}\mathcal{L}_B(x\mid\mathbf{w}).
\label{eq:mobius-residual}
\end{equation}
This is Mobius inversion on the Boolean lattice \citep{rota1964}.

\begin{proposition}[Residual identification]
Assume the hierarchical expansion
\begin{equation}
\mathcal{L}_A(x\mid\mathbf{w})=\sum_{\emptyset\neq B\subseteq A}\Phi_B(x\mid\mathbf{w}).
\label{eq:hierarchical-expansion}
\end{equation}
Then for every non-empty $A\subseteq[m]$,
\begin{equation}
\Delta_A(x\mid\mathbf{w})=\Phi_A(x\mid\mathbf{w}).
\end{equation}
Conversely,
\begin{equation}
\mathcal{L}_A(x\mid\mathbf{w})=\sum_{\emptyset\neq B\subseteq A}\Delta_B(x\mid\mathbf{w}).
\end{equation}
\end{proposition}

\begin{proof}
Substituting \eqref{eq:hierarchical-expansion} into \eqref{eq:mobius-residual} gives
\[
\Delta_A=\sum_{B\subseteq A}(-1)^{|A|-|B|}\sum_{\emptyset\neq C\subseteq B}\Phi_C
=\sum_{\emptyset\neq C\subseteq A}\Phi_C\sum_{B:C\subseteq B\subseteq A}(-1)^{|A|-|B|}.
\]
For $C=A$ the inner sum equals $1$. For $C\subsetneq A$, write $B=C\cup D$ with $D\subseteq A\setminus C$. The inner sum becomes
\[
\sum_{D\subseteq A\setminus C}(-1)^{|A|-|C|-|D|}=(1-1)^{|A|-|C|}=0.
\]
Thus only $\Phi_A$ survives. The inverse formula is the standard inverse of Mobius inversion on the Boolean lattice.
\end{proof}

For $A=\{i,j,k\}$ this gives
\begin{align}
\Delta_{ijk}
&=\mathcal{L}_{ijk}-\mathcal{L}_{ij}-\mathcal{L}_{ik}-\mathcal{L}_{jk}
+\mathcal{L}_i+\mathcal{L}_j+\mathcal{L}_k,
\label{eq:three-word-residual}
\end{align}
with arguments suppressed. The triple residual is the first residual that cannot be reduced to isolated words or pairwise compatibility. This is the precise mathematical form of the three-word problem.

\subsection{Order spectra}

Residuals become empirically useful when equipped with norms. Let $\|f\|_{\mathcal{H}}$ be a chosen function norm, such as an $L^2$ norm over $S$, an empirical norm over sampled semantic states, or a reproducing-kernel norm. Define the order-$k$ residual mass
\begin{equation}
R_k(\mathbf{w})=\left(\sum_{\substack{A\in\mathcal{K}_p(\mathbf{w})\\ |A|=k}}\|\Delta_A(\cdot\mid\mathbf{w})\|_{\mathcal{H}}^2\right)^{1/2}.
\label{eq:order-mass}
\end{equation}
The vector
\begin{equation}
\operatorname{Spec}_p(\mathbf{w})=(R_1(\mathbf{w}),\ldots,R_p(\mathbf{w}))
\label{eq:order-spectrum}
\end{equation}
is the \emph{semantic order spectrum} of the sentence under the model.

The order spectrum is a new diagnostic object. Additive semantic models predict concentration at $R_1$. Pairwise compositional models predict that most residual mass lies in $R_1$ and $R_2$. SFT predicts that idioms, metaphors, coercions, and event construals should produce systematically higher $R_3$ or higher-order mass than literal compositional controls, after controlling for sentence length and lexical frequency.

\section{Interpretation as stabilized energy minimization}
\label{sec:stabilization}

Given a linguistic field $\mathcal{L}_p(\cdot\mid\mathbf{w})$, SFT defines interpretation as stabilization in an energy landscape. Let
\begin{equation}
U(x\mid\mathbf{w})=\frac{\beta}{2}\|x\|^2-\mathcal{L}_p(x\mid\mathbf{w}),\qquad \beta>0.
\label{eq:energy}
\end{equation}
A stabilized interpretation is a local or global minimizer of $U$:
\begin{equation}
x^*(\mathbf{w})\in\argmin_{x\in S}U(x\mid\mathbf{w}).
\label{eq:stabilized-state}
\end{equation}
The quadratic term prevents unbounded drift and can be replaced by a more structured prior when $S$ has known geometry.

\begin{proposition}[Existence]
Assume $S=\mathbb{R}^n$, $p<\infty$, $\mathcal{K}_p(\mathbf{w})$ is finite, and the deformed lexical fields are bounded and continuous. Then $U(\cdot\mid\mathbf{w})$ is coercive and attains a global minimum.
\end{proposition}

\begin{proof}
Finite products and finite sums of bounded continuous functions are bounded and continuous, so $\mathcal{L}_p(\cdot\mid\mathbf{w})$ is bounded and continuous. The term $\frac{\beta}{2}\|x\|^2$ diverges to $+\infty$ as $\|x\|\to\infty$. Hence $U$ is continuous and coercive on $\mathbb{R}^n$, and therefore attains a global minimum.
\end{proof}

The gradient flow is
\begin{equation}
\frac{dx}{dt}=-\nabla U(x\mid\mathbf{w})=\nabla\mathcal{L}_p(x\mid\mathbf{w})-\beta x.
\label{eq:gradient-flow}
\end{equation}
A discrete inference scheme is
\begin{equation}
x^{(t+1)}=x^{(t)}-\eta\nabla U(x^{(t)}\mid\mathbf{w})
=x^{(t)}+\eta\big(\nabla\mathcal{L}_p(x^{(t)}\mid\mathbf{w})-\beta x^{(t)}\big).
\label{eq:discrete-update}
\end{equation}

\begin{proposition}[Descent]
If $\nabla U(\cdot\mid\mathbf{w})$ is $L_U$-Lipschitz and $0<\eta<2/L_U$, then the update \eqref{eq:discrete-update} satisfies
\begin{equation}
U(x^{(t+1)}\mid\mathbf{w})\le U(x^{(t)}\mid\mathbf{w})-
\eta\left(1-\frac{L_U\eta}{2}\right)\|\nabla U(x^{(t)}\mid\mathbf{w})\|^2.
\end{equation}
\end{proposition}

\begin{proof}
This is the standard smooth-objective descent inequality applied to gradient descent on $U$.
\end{proof}

\subsection{Uniqueness and multistability}

SFT should not always force a unique interpretation. Ambiguity and polysemy may correspond to multiple local minima. Still, a uniqueness condition is useful.

\begin{proposition}[Strong-convexity sufficient condition]
Assume $\mathcal{L}_p$ is twice differentiable and there exists $M\ge0$ such that
\[
\nabla^2\mathcal{L}_p(x\mid\mathbf{w})\preceq M I
\quad\text{for all }x\in\mathbb{R}^n.
\]
If $\beta>M$, then $U(\cdot\mid\mathbf{w})$ is $(\beta-M)$-strongly convex and has a unique global minimizer.
\end{proposition}

\begin{proof}
Since $\nabla^2 U(x\mid\mathbf{w})=\beta I-\nabla^2\mathcal{L}_p(x\mid\mathbf{w})$, the assumed bound implies
\[
\nabla^2 U(x\mid\mathbf{w})\succeq(\beta-M)I.
\]
Thus $U$ is strongly convex. A strongly convex coercive function has a unique global minimizer.
\end{proof}

When $\beta\le M$, several minima may exist. Rather than being a defect, this can model multiple stabilized readings. The energy gap between minima and the volume of their basins become semantic diagnostics.

\subsection{Stability under perturbation}

Let $U$ and $\widehat U$ be two sentence-conditioned energies, for example produced by two nearby parameter settings or two paraphrastic sentences. Suppose both are $\mu$-strongly convex with minimizers $x^*$ and $\widehat x^*$.

\begin{proposition}[Perturbation stability]
If
\[
\sup_x\|\nabla U(x)-\nabla\widehat U(x)\|\le \varepsilon
\]
and $\widehat U$ is $\mu$-strongly convex, then
\begin{equation}
\|x^*-\widehat x^*\|\le \frac{\varepsilon}{\mu}.
\end{equation}
\end{proposition}

\begin{proof}
Because $\nabla U(x^*)=0$ and $\nabla\widehat U(\widehat x^*)=0$,
\[
\|\nabla\widehat U(x^*)-\nabla\widehat U(\widehat x^*)\|
=\|\nabla\widehat U(x^*)\|
=\|\nabla\widehat U(x^*)-\nabla U(x^*)\|\le\varepsilon.
\]
For a $\mu$-strongly convex differentiable function, its gradient is $\mu$-strongly monotone, implying
\[
\|\nabla\widehat U(x^*)-\nabla\widehat U(\widehat x^*)\|\ge \mu\|x^*-\widehat x^*\|.
\]
Combining the inequalities gives the result.
\end{proof}

This proposition gives a principled sense in which stabilized interpretations vary continuously with small field perturbations when the energy basin is well conditioned.

\subsection{Probabilistic interpretation}

For temperature $\tau>0$, define
\begin{equation}
p_\tau(x\mid\mathbf{w})=Z(\mathbf{w})^{-1}\exp\!\left(-\frac{U(x\mid\mathbf{w})}{\tau}\right),
\qquad
Z(\mathbf{w})=\int_{\mathbb{R}^n}\exp\!\left(-\frac{U(x\mid\mathbf{w})}{\tau}\right)dx.
\end{equation}
Because $U$ is coercive under the conditions above, $Z(\mathbf{w})$ is finite for Gaussian-field instantiations. The zero-temperature limit emphasizes global minimizers, while finite temperature represents uncertainty, ambiguity, and graded interpretive alternatives. This connects SFT to energy-based modeling \citep{lecun2006} without identifying linguistic meaning with the energy model itself.

\section{Learning and computational instantiation}
\label{sec:learning}

Let $\Theta$ denote lexical, contextual deformation, interaction, and energy parameters. Given supervised data
\[
\mathcal{D}=\{(\mathbf{w}^{(r)},y^{(r)})\}_{r=1}^N,
\]
define $x^*(\mathbf{w})$ by $T$ unrolled steps of \eqref{eq:discrete-update} or by an implicit optimization layer. A prediction head $F_\psi$ consumes $x^*(\mathbf{w})$, the order spectrum, or energy-derived features. A general training objective is
\begin{equation}
\mathcal{J}(\Theta,\psi)=\sum_{r=1}^N\ell\big(F_\psi(\chi(\mathbf{w}^{(r)})),y^{(r)}\big)
+\alpha\Omega_{\Sigma}(\Theta)+\rho\Omega_{\mathrm{int}}(\Theta)+\zeta\Omega_{\mathrm{stab}}(\Theta),
\label{eq:training-objective}
\end{equation}
where
\[
\chi(\mathbf{w})=\big[x^*(\mathbf{w});\operatorname{Spec}_p(\mathbf{w});U(x^*(\mathbf{w})\mid\mathbf{w})\big].
\]
The regularizer $\Omega_{\Sigma}$ controls covariance degeneracy, $\Omega_{\mathrm{int}}$ encourages sparse higher-order interactions, and $\Omega_{\mathrm{stab}}$ penalizes unstable inference trajectories or excessively flat minima.

The computational bottleneck is the interaction complex. A complete order-$p$ expansion has $O(m^p)$ terms. A scientifically useful implementation should therefore report not only accuracy but also the selection rule for $\mathcal{K}_p$, the retained interaction density, and the contribution of each order spectrum component.

\section{Relation to transformer representations}
\label{sec:transformers}

Transformers are not SFT, and SFT is not an alternative derivation of the transformer architecture. The relation is one of effective description. A self-attention layer computes
\begin{align}
\alpha_{ij}^{(\ell)}&=\operatorname{softmax}_j\left(\frac{(W_Q^{(\ell)}h_i^{(\ell)})^\top(W_K^{(\ell)}h_j^{(\ell)})}{\sqrt{d_k}}\right),\\
h_i^{(\ell+1)}&=f\left(h_i^{(\ell)},\sum_j\alpha_{ij}^{(\ell)}W_V^{(\ell)}h_j^{(\ell)}\right).
\end{align}
This mechanism performs contextual reweighting and nonlinear recombination over token states \citep{vaswani2017}. Probing studies suggest that transformer representations encode syntactic and semantic regularities, though attention weights should not be naively equated with explanations \citep{clark2019,hewitt2019,tenney2019,jain2019,rogers2020}.

SFT can be placed on top of or alongside such representations in at least three ways.

\paragraph{SFT as a semantic probe.}
A pretrained encoder supplies $q_\phi(\mathbf{w})$ and token features. SFT parameters are fitted while the encoder is frozen. The question is whether explicit field overlap, residual mass, and stabilization features explain behavior beyond standard embedding similarities.

\paragraph{SFT as an interpretable head.}
A downstream model uses $x^*(\mathbf{w})$ and $\operatorname{Spec}_p(\mathbf{w})$ as interpretable features. The model can report which interaction subsets contributed to a decision and whether the decision depends on higher-order residuals.

\paragraph{SFT as a regularizer.}
During fine-tuning, one can encourage stable fields, sparse interaction complexes, or controlled order spectra. This does not force a transformer to become an SFT model; it imposes a semantic bias on a flexible representation learner.

The main conceptual difference is that transformers produce contextual states directly, whereas SFT decomposes semantic organization into type-level fields, contextual deformation, explicit subset interactions, and energy stabilization. This decomposition is the source of SFT's testable claims.

\section{Evaluation strategy and empirical predictions}
\label{sec:evaluation}

A mature follow-up paper must make clear what could count against the theory. SFT becomes scientifically meaningful when it predicts measurable differences between model orders and semantic constructions.

\subsection{Interaction-order ablation}

Train or fit variants with $p=1$, $p=2$, and $p=3$ while holding the encoder and parameter budget as fixed as possible. Evaluate on phrase similarity, sentence similarity, natural-language inference, and controlled idiom/metaphor datasets. Benchmarks such as SICK, SNLI, MultiNLI, GLUE, and SuperGLUE provide starting points for such comparisons \citep{marelli2014,bowman2015,williams2018,wang2018,wang2019}. The key prediction is not simply that higher order improves performance; flexible models often improve when more parameters are added. The stronger prediction is that $R_3$ should selectively explain cases involving non-additive joint construal.

\subsection{Residual diagnostics}

Construct minimal triples such as literal controls, metaphoric triples, coercion triples, and idiomatic triples. For each sentence, compute $R_1,R_2,R_3$ and compare the normalized third-order ratio
\begin{equation}
\rho_3(\mathbf{w})=\frac{R_3(\mathbf{w})}{R_1(\mathbf{w})+R_2(\mathbf{w})+R_3(\mathbf{w})+\epsilon}.
\end{equation}
SFT predicts that $\rho_3$ should be higher for constructions whose interpretation depends on joint three-way interaction than for matched literal controls.

\subsection{Field-overlap predictions}

The compatibility factor $\kappa_A$ and overlap integrals give predictions for lexical substitution, priming, and semantic acceptability. If two words have high field overlap but differ in covariance, they may be near in center-based similarity but behave differently under interaction. This supplies a test that distinguishes field semantics from point-vector semantics.

\subsection{Stabilization diagnostics}

Energy-based interpretation yields additional observables: convergence time, gradient norm decay, energy gaps, basin sensitivity, and finite-temperature entropy. Sentences with unstable or ambiguous readings should exhibit flatter minima, smaller energy gaps, or multiple local attractors. These diagnostics can be compared against human ambiguity judgments or model uncertainty.

\subsection{Multilingual structure}

If translation equivalents share partial field geometry but differ in covariance and interaction coefficients, multilingual encoders should permit partial alignment of centers while preserving language-specific deformation patterns. This predicts a middle position between strict semantic identity across languages and complete incommensurability.

\section{Scope and limitations}
\label{sec:limits}

SFT is a formal hypothesis, not a completed semantics. Several limitations are essential.

First, the semantic space $S$ is model-dependent. Different encoders, training corpora, and supervision regimes may induce different geometries. SFT therefore requires identifiability analysis: which field parameters are stable under reparameterization, and which are artifacts of the coordinate system?

Second, higher-order interaction is expensive. A full third-order model is already cubic in sentence length. Any practical implementation must use sparsity, syntax, local windows, or learned pruning. Such approximations are not merely engineering details; they determine which interactions the theory can detect.

Third, SFT does not solve grounding or normativity. It can describe stable representational regularities induced by language use, but it does not replace the public practices through which meanings are taught, contested, and corrected.

Fourth, relation to transformers remains methodological rather than identificatory. A transformer may contain patterns that can be approximated by SFT, but no attention head or hidden layer should be declared a semantic field without explicit modeling and validation.

Fifth, the strong philosophical phrase ``intrinsic mathematical structure'' should be treated as a research hypothesis. The scientifically safer formulation is that language use induces mathematically tractable regularities at a representational level. The stronger metaphysical reading is not required for the model class.

\section{Conclusion}

This paper has reconstructed Semantic Field Theory as an evolving formal program. The historical origin lies in a philosophical critique of strong anti-formalist interpretations of language games. The underlying intuition of lexical semantic fields was present from the beginning, although only later reformulated in explicit mathematical and computational terms. The contemporary mathematical form is a model class in which lexical items are lifted to distributed fields, contexts deform those fields, utterances activate subset-indexed interaction complexes, irreducible higher-order effects are isolated by Mobius residuals, and interpretation is computed as stabilization in an energy landscape.

The paper's main scientific claim is limited but testable: a tractable level of semantic organization may be captured by field geometry, higher-order residual structure, and energy-based inference. The Gaussian product closure result makes interaction geometry explicit. The order spectrum turns the three-word problem into a measurable diagnostic. The energy formulation gives SFT an implementable inference procedure with existence and stability guarantees. The worked summer-day example illustrates how this formal apparatus can be reduced to a minimal executable pipeline without confusing the toy instantiation with a trained semantic model.

The resulting theory does not deny that meaning is public, social, and normative. Instead, it proposes that public language use can leave stable mathematical structure without being reducible to that structure. In that sense, SFT aims to occupy a middle position: stronger than metaphorical talk about semantic fields, weaker than a complete reduction of language to vectors, and sufficiently formal to be estimated, ablated, and criticized.

\section*{Acknowledgments}

The author thanks TWT GmbH Science \& Innovation and NIKI Digital Engineering for support. He also
thanks S. Katsioli for
helpful discussions.

\appendix

\section{Supplementary calculations}

\subsection{Gradients of Gaussian fields}

For a deformed Gaussian field
\[
\widetilde L_i(x)=\tilde a_i\exp\left[-\frac12(x-\tilde c_i)^\top P_i(x-\tilde c_i)\right],
\]
where $P_i=\widetilde\Sigma_i^{-1}$,
\begin{equation}
\nabla\widetilde L_i(x)=-\widetilde L_i(x)P_i(x-\tilde c_i).
\end{equation}
For an interaction term $\Phi_A(x)=\lambda_A\prod_{i\in A}\widetilde L_i(x)$,
\begin{equation}
\nabla\Phi_A(x)=\Phi_A(x)\sum_{i\in A}\big[-P_i(x-\tilde c_i)\big].
\end{equation}
Thus
\begin{equation}
\nabla U(x\mid\mathbf{w})=\beta x-\sum_{A\in\mathcal{K}_p(\mathbf{w})}\Phi_A(x\mid\mathbf{w})\sum_{i\in A}\big[-P_i(x-\tilde c_i)\big].
\end{equation}

\subsection{Closed-form pairwise values for the summer-day example}

For the isotropic summer-day example in Section~\ref{sec:example}, pairwise interaction tensions reduce to
\begin{equation}
T_{ij}=\frac12\|c_i-c_j\|^2,
\qquad
\kappa_{ij}=\exp(-T_{ij}/2).
\end{equation}
The resulting approximate values are
\begin{center}
\begin{tabular}{@{}lcc@{}}
\toprule
Pair & $T_{ij}$ & $\kappa_{ij}$ \\
\midrule
helios--thalassa & $0.1375$ & $0.9336$ \\
helios--zesti & $0.0357$ & $0.9823$ \\
thalassa--zesti & $0.2667$ & $0.8752$ \\
\bottomrule
\end{tabular}
\end{center}
These numbers are not linguistic measurements. They are sanity checks for the toy field geometry: the pair \emph{helios--zesti} is closest in thermal intensity, while \emph{thalassa--zesti} has greater tension because the marine and heat dimensions pull in different directions.

\subsection{\texorpdfstring{Closed-form $L^2$ norm for Gaussian interactions}{Closed-form L2 norm for Gaussian interactions}}

If
\[
f_A(x)=C_A\exp\left[-\frac12(x-\mu_A)^\top P_A(x-\mu_A)\right],
\]
with $P_A\succ0$, then
\begin{equation}
\|f_A\|_{L^2(\mathbb{R}^n)}^2
=C_A^2\int_{\mathbb{R}^n}\exp\left[-(x-\mu_A)^\top P_A(x-\mu_A)\right]dx
=C_A^2\pi^{n/2}|P_A|^{-1/2}.
\end{equation}
This permits exact residual norms when residuals consist of single Gaussian interaction terms. When residuals are sums of Gaussian terms, pairwise cross-integrals also have closed forms by the same product rule.

\subsection{Algorithmic schema}

\noindent\textbf{Input:} sentence $\mathbf{w}=(w_1,\ldots,w_m)$; order $p$; parameters $\Theta$; step size $\eta$; tolerance $\epsilon$.\\
\textbf{Output:} stabilized state $x^*(\mathbf{w})$, order spectrum $\operatorname{Spec}_p(\mathbf{w})$.

\begin{enumerate}
\item Map each type $t_i$ to base field parameters $(a_{t_i},c_{t_i},\Sigma_{t_i})$.
\item Compute contextual representation $q_\phi(\mathbf{w})$.
\item Deform each lexical field to $\widetilde L_i(\cdot\mid\mathbf{w})$.
\item Select interaction complex $\mathcal{K}_p(\mathbf{w})$.
\item Compute coefficients $\lambda_A(\mathbf{w})$ for $A\in\mathcal{K}_p(\mathbf{w})$.
\item Construct $\mathcal L_p(x\mid\mathbf{w})=\sum_A\Phi_A(x\mid\mathbf{w})$ and energy $U(x\mid\mathbf{w})$.
\item Initialize $x^{(0)}$ by a pooled center, random restart, or encoder projection.
\item Iterate $x^{(t+1)}=x^{(t)}-\eta\nabla U(x^{(t)}\mid\mathbf{w})$ until convergence.
\item Compute residuals $\Delta_A$ and order spectrum $\operatorname{Spec}_p(\mathbf{w})$.
\item Return $x^*(\mathbf{w})=x^{(t+1)}$ and diagnostics.
\end{enumerate}


\begin{thebibliography}{}

\bibitem[Bender and Koller(2020)]{bender2020}
Emily M. Bender and Alexander Koller. 2020.
\newblock Climbing towards NLU: On meaning, form, and understanding in the age of data.
\newblock In \emph{Proceedings of the 58th Annual Meeting of the Association for Computational Linguistics}, pages 5185--5198.

\bibitem[Bender et~al.(2021)Bender, Gebru, McMillan-Major, and Shmitchell]{bender2021}
Emily M. Bender, Timnit Gebru, Angelina McMillan-Major, and Shmargaret Shmitchell. 2021.
\newblock On the dangers of stochastic parrots: Can language models be too big?
\newblock In \emph{Proceedings of the 2021 ACM Conference on Fairness, Accountability, and Transparency}, pages 610--623.

\bibitem[Bommasani et~al.(2021)Bommasani, Hudson, Adeli, Altman, Arora, von Arx, Bernstein, Bohg, Bosselut, Brunskill et~al.]{bommasani2021}
Rishi Bommasani, Drew A. Hudson, Ehsan Adeli, Russ Altman, Simran Arora, Sydney von Arx, Michael S. Bernstein, Jeannette Bohg, Anja Butterfield, and many others. 2021.
\newblock On the opportunities and risks of foundation models.
\newblock \emph{arXiv preprint}, arXiv:2108.07258.

\bibitem[Bowman et~al.(2015)Bowman, Angeli, Potts, and Manning]{bowman2015}
Samuel R. Bowman, Gabor Angeli, Christopher Potts, and Christopher D. Manning. 2015.
\newblock A large annotated corpus for learning natural language inference.
\newblock In \emph{Proceedings of EMNLP}, pages 632--642.

\bibitem[Brown et~al.(2020)Brown, Mann, Ryder, Subbiah, Kaplan, Dhariwal, Neelakantan, Shyam, Sastry, Askell et~al.]{brown2020}
Tom B. Brown, Benjamin Mann, Nick Ryder, Melanie Subbiah, Jared Kaplan, Prafulla Dhariwal, Arvind Neelakantan, Pranav Shyam, Girish Sastry, Amanda Askell, and others. 2020.
\newblock Language models are few-shot learners.
\newblock In \emph{Advances in Neural Information Processing Systems}, 33:1877--1901.

\bibitem[Clark et~al.(2019)Clark, Khandelwal, Levy, and Manning]{clark2019}
Kevin Clark, Urvashi Khandelwal, Omer Levy, and Christopher D. Manning. 2019.
\newblock What does BERT look at? An analysis of BERT's attention.
\newblock In \emph{Proceedings of the 2019 ACL Workshop BlackboxNLP}, pages 276--286.

\bibitem[Coecke et~al.(2010)Coecke, Sadrzadeh, and Clark]{coecke2010}
Bob Coecke, Mehrnoosh Sadrzadeh, and Stephen Clark. 2010.
\newblock Mathematical foundations for a compositional distributional model of meaning.
\newblock \emph{Linguistic Analysis}, 36(1--4):345--384.

\bibitem[Devlin et~al.(2019)Devlin, Chang, Lee, and Toutanova]{devlin2019}
Jacob Devlin, Ming-Wei Chang, Kenton Lee, and Kristina Toutanova. 2019.
\newblock BERT: Pre-training of deep bidirectional transformers for language understanding.
\newblock In \emph{Proceedings of NAACL-HLT}, pages 4171--4186.

\bibitem[Ethayarajh(2019)]{ethayarajh2019}
Kawin Ethayarajh. 2019.
\newblock How contextual are contextualized word representations? Comparing the geometry of BERT, ELMo, and GPT-2 embeddings.
\newblock In \emph{Proceedings of EMNLP-IJCNLP}, pages 55--65.

\bibitem[Firth(1957)]{firth1957}
J. R. Firth. 1957.
\newblock \emph{Papers in Linguistics 1934--1951}.
\newblock Oxford University Press.

\bibitem[Grefenstette et~al.(2011)Grefenstette, Sadrzadeh, Clark, Coecke, and Pulman]{grefenstette2011}
Edward Grefenstette, Mehrnoosh Sadrzadeh, Stephen Clark, Bob Coecke, and Stephen Pulman. 2011.
\newblock Concrete sentence spaces for compositional distributional models of meaning.
\newblock In \emph{Proceedings of IWCS 2011}.

\bibitem[Harris(1954)]{harris1954}
Zellig S. Harris. 1954.
\newblock Distributional structure.
\newblock \emph{WORD}, 10(2--3):146--162.

\bibitem[Hewitt and Manning(2019)]{hewitt2019}
John Hewitt and Christopher D. Manning. 2019.
\newblock A structural probe for finding syntax in word representations.
\newblock In \emph{Proceedings of NAACL-HLT}, pages 4129--4138.

\bibitem[Jain and Wallace(2019)]{jain2019}
Sarthak Jain and Byron C. Wallace. 2019.
\newblock Attention is not explanation.
\newblock In \emph{Proceedings of NAACL-HLT}, pages 3543--3556.

\bibitem[Kaplan et~al.(2020)Kaplan, McCandlish, Henighan, Brown, Chess, Child, Gray, Radford, Wu, and Amodei]{kaplan2020}
Jared Kaplan, Sam McCandlish, Tom Henighan, Tom B. Brown, Benjamin Chess, Rewon Child, Scott Gray, Alec Radford, Jeffrey Wu, and Dario Amodei. 2020.
\newblock Scaling laws for neural language models.
\newblock \emph{arXiv preprint}, arXiv:2001.08361.

\bibitem[Kripke(1982)]{kripke1982}
Saul A. Kripke. 1982.
\newblock \emph{Wittgenstein on Rules and Private Language: An Elementary Exposition}.
\newblock Harvard University Press.

\bibitem[LeCun et~al.(2006)LeCun, Chopra, Hadsell, Ranzato, and Huang]{lecun2006}
Yann LeCun, Sumit Chopra, Raia Hadsell, Marc'Aurelio Ranzato, and Fu-Jie Huang. 2006.
\newblock A tutorial on energy-based learning.
\newblock In G. Bakir, T. Hofmann, B. Scholkopf, A. Smola, B. Taskar, and S. Vishwanathan, editors, \emph{Predicting Structured Data}, MIT Press.

\bibitem[Marelli et~al.(2014)Marelli, Menini, Baroni, Bentivogli, Bernardi, and Zamparelli]{marelli2014}
Marco Marelli, Stefano Menini, Marco Baroni, Luisa Bentivogli, Raffaella Bernardi, and Roberto Zamparelli. 2014.
\newblock A SICK cure for the evaluation of compositional distributional semantic models.
\newblock In \emph{Proceedings of LREC}, pages 216--223.

\bibitem[Mikolov et~al.(2013)Mikolov, Yih, and Zweig]{mikolov2013}
Tomas Mikolov, Wen-tau Yih, and Geoffrey Zweig. 2013.
\newblock Linguistic regularities in continuous space word representations.
\newblock In \emph{Proceedings of NAACL-HLT}, pages 746--751.

\bibitem[Mitchell and Lapata(2008)]{mitchell2008}
Jeff Mitchell and Mirella Lapata. 2008.
\newblock Vector-based models of semantic composition.
\newblock In \emph{Proceedings of ACL-08: HLT}, pages 236--244.

\bibitem[Mitchell and Lapata(2010)]{mitchell2010}
Jeff Mitchell and Mirella Lapata. 2010.
\newblock Composition in distributional models of semantics.
\newblock \emph{Cognitive Science}, 34(8):1388--1429.

\bibitem[Pennington et~al.(2014)Pennington, Socher, and Manning]{pennington2014}
Jeffrey Pennington, Richard Socher, and Christopher D. Manning. 2014.
\newblock GloVe: Global vectors for word representation.
\newblock In \emph{Proceedings of EMNLP}, pages 1532--1543.

\bibitem[Peters et~al.(2018)Peters, Neumann, Iyyer, Gardner, Clark, Lee, and Zettlemoyer]{peters2018}
Matthew E. Peters, Mark Neumann, Mohit Iyyer, Matt Gardner, Christopher Clark, Kenton Lee, and Luke Zettlemoyer. 2018.
\newblock Deep contextualized word representations.
\newblock In \emph{Proceedings of NAACL-HLT}, pages 2227--2237.

\bibitem[Rogers et~al.(2020)Rogers, Kovaleva, and Rumshisky]{rogers2020}
Anna Rogers, Olga Kovaleva, and Anna Rumshisky. 2020.
\newblock A primer in BERTology: What we know about how BERT works.
\newblock \emph{Transactions of the Association for Computational Linguistics}, 8:842--866.

\bibitem[Rota(1964)]{rota1964}
Gian-Carlo Rota. 1964.
\newblock On the foundations of combinatorial theory I. Theory of Mobius functions.
\newblock \emph{Zeitschrift fuer Wahrscheinlichkeitstheorie und Verwandte Gebiete}, 2:340--368.

\bibitem[Tenney et~al.(2019)Tenney, Das, and Pavlick]{tenney2019}
Ian Tenney, Dipanjan Das, and Ellie Pavlick. 2019.
\newblock BERT rediscovers the classical NLP pipeline.
\newblock In \emph{Proceedings of ACL}, pages 4593--4601.

\bibitem[Turney and Pantel(2010)]{turney2010}
Peter D. Turney and Patrick Pantel. 2010.
\newblock From frequency to meaning: Vector space models of semantics.
\newblock \emph{Journal of Artificial Intelligence Research}, 37:141--188.

\bibitem[Vartziotis(2012)]{vartziotis2012}
Dimitris Vartziotis. 2012.
\newblock \emph{Scholia se stochasmous tou Ludwig Wittgenstein [Scholia on Reflections of Ludwig Wittgenstein]}.
\newblock Lefki Selida.

\bibitem[Vartziotis(2017)]{vartziotis2017}
Dimitris Vartziotis. 2017.
\newblock \emph{Kommentare zu Wittgensteins Zitaten}.
\newblock Literareon -- Utz Verlag.

\bibitem[Vartziotis(in press-a)]{vartziotisUpcoming1}
Dimitris Vartziotis.
In press.
\newblock \emph{Wittgenstein and the End of Language Games: Mathematical Semantics, Field Theory, and the Age of Predictive Language Models}.
\newblock Literareon -- Utz Verlag.

\bibitem[Vartziotis(in press-b)]{vartziotisUpcoming2}
Dimitris Vartziotis.
In press.
\newblock \emph{The Revolution of LLMs: Philosophical and Mathematical Conjugations}.
\newblock Literareon -- Utz Verlag.

\bibitem[Vartziotis(2026)]{vartziotis2026a}
Dimitris Vartziotis. 2026.
\newblock Language as mathematical structure: Examining Semantic Field Theory against language games.
\newblock \emph{arXiv preprint}, arXiv:2601.00448.

\bibitem[Vaswani et~al.(2017)Vaswani, Shazeer, Parmar, Uszkoreit, Jones, Gomez, Kaiser, and Polosukhin]{vaswani2017}
Ashish Vaswani, Noam Shazeer, Niki Parmar, Jakob Uszkoreit, Llion Jones, Aidan N. Gomez, Lukasz Kaiser, and Illia Polosukhin. 2017.
\newblock Attention is all you need.
\newblock In \emph{Advances in Neural Information Processing Systems}, 30:5998--6008.

\bibitem[Wang et~al.(2018)Wang, Singh, Michael, Hill, Levy, and Bowman]{wang2018}
Alex Wang, Amanpreet Singh, Julian Michael, Felix Hill, Omer Levy, and Samuel R. Bowman. 2018.
\newblock GLUE: A multi-task benchmark and analysis platform for natural language understanding.
\newblock In \emph{Proceedings of the BlackboxNLP Workshop}, pages 353--355.

\bibitem[Wang et~al.(2019)Wang, Pruksachatkun, Nangia, Singh, Michael, Hill, Levy, and Bowman]{wang2019}
Alex Wang, Yada Pruksachatkun, Nikita Nangia, Amanpreet Singh, Julian Michael, Felix Hill, Omer Levy, and Samuel R. Bowman. 2019.
\newblock SuperGLUE: A stickier benchmark for general-purpose language understanding systems.
\newblock In \emph{Advances in Neural Information Processing Systems}, 32.

\bibitem[Williams et~al.(2018)Williams, Nangia, and Bowman]{williams2018}
Adina Williams, Nikita Nangia, and Samuel R. Bowman. 2018.
\newblock A broad-coverage challenge corpus for sentence understanding through inference.
\newblock In \emph{Proceedings of NAACL-HLT}, pages 1112--1122.

\bibitem[Wittgenstein(1953)]{wittgenstein1953}
Ludwig Wittgenstein. 1953.
\newblock \emph{Philosophical Investigations}.
\newblock Blackwell.

\end{thebibliography}
\end{document}